\documentclass{article}

\usepackage[preprint]{neurips_2023}

\usepackage[utf8]{inputenc} 
\usepackage[T1]{fontenc}    
\usepackage{hyperref}       
\usepackage{url}            
\usepackage{booktabs}       
\usepackage{amsfonts}       
\usepackage{nicefrac}       
\usepackage{microtype}      
\usepackage{xcolor}         
\usepackage[small]{caption}
\usepackage{subcaption}
\usepackage{graphicx}  
\usepackage{mathptmx}
\usepackage{mathtools}
\usepackage{amsmath}
\usepackage{amsthm}
\usepackage{physics}
\usepackage[shortlabels]{enumitem}
\usepackage{thmtools, thm-restate}
\usepackage{algorithm}
\usepackage{algpseudocode}
\usepackage{multicol}
\usepackage{booktabs} 
\usepackage{makecell} 
\usepackage[flushleft]{threeparttable} 
\usepackage{verbatimbox}

\DeclareMathOperator{\Agg}{A}
\DeclareMathOperator{\tnorm}{T}
\DeclareMathOperator{\snorm}{S}
\DeclareMathOperator{\fzimp}{I}
\DeclareMathOperator{\fzneg}{N}

\DeclareMathOperator{\argmax}{argmax}
\DeclareMathOperator{\argmin}{argmin}

\DeclareMathOperator{\mean}{mean}
\DeclareMathOperator*{\Sat}{SatAgg}

\DeclareMathOperator{\softmax}{\sigma}
\DeclareMathOperator{\sigmoid}{S}

\def\G{\mathcal{G}}
\newcommand{\imp}{\rightarrow}
\newcommand{\K}{\mathcal{K}}

\newcommand{\R}{\mathbb{R}}

\def\btheta{\boldsymbol{\theta}}
\newcommand{\Gt}{\mathcal{G}_{\btheta}}

\DeclareMathOperator{\Tprod}{T_{P}}
\DeclareMathOperator{\Tluk}{T_{L}}
\DeclareMathOperator{\Smax}{S_{M} }
\DeclareMathOperator{\Sprod}{S_{P}}

\DeclareMathOperator{\Ns}{N_{S}}

\DeclareMathOperator{\Atprod}{A_{\Tprod}}
\DeclareMathOperator{\Asmax}{A_{\Smax}}

\newcommand{\logLTN}{\textrm{logLTN}}
\newcommand{\LTN}{\textrm{LTN}}

\DeclareMathOperator{\LSE}{LSE}
\DeclareMathOperator{\LME}{LME}

\newlength\myheight
\newlength\mydepth
\settototalheight\myheight{Xygp}
\newcommand*\inlinegraphics[1]{%
  \settototalheight\myheight{Xygp}%
  \settodepth\mydepth{Xygp}%
  \raisebox{-\mydepth}{\includegraphics[height=\myheight]{#1}}%
}
\newtheorem{theorem}{Theorem}
\newtheorem{corollary}{Corollary}[theorem]
\theoremstyle{definition}
\newtheorem{example}{Example}

\title{logLTN: Differentiable Fuzzy Logic in the Logarithm Space}

\author{%
  Samy Badreddine\\
  Sony AI\\
  Tokyo\\
  Japan \\
  \texttt{samy.badreddine@sony.com} \\
  \And
  Luciano Serafini\\
  Fondazione Bruno Kessler\\
  Trento\\
  Italy \\
  \texttt{serafini@fbk.eu} \\
  \And
  Michael Spranger\\
  Sony AI\\
  Tokyo\\
  Japan \\
  \texttt{michael.spranger@sony.com} \\
}

\begin{document}
\maketitle

\begin{abstract}
    The AI community is increasingly focused on merging logic with deep learning to create Neuro-Symbolic (NeSy) paradigms and assist neural approaches with symbolic knowledge.
    A significant trend in the literature involves integrating axioms and facts in loss functions by grounding logical symbols with neural networks 
    and operators with fuzzy semantics. 
    Logic Tensor Networks (LTN) is one of the leading representatives in this category, known for its simplicity, efficiency, and versatility.
    However, it has been previously shown that not all fuzzy operators perform equally when applied in a differentiable setting. 
    Researchers have proposed several configurations of operators, trading off between effectiveness, numerical stability, and generalization to different formulas. 
    This paper presents a configuration of fuzzy operators for grounding formulas end-to-end in the logarithm space.
    Our goal is to develop a configuration that is more effective than previous proposals, able to handle any formula, and numerically stable. 
    To achieve this, we propose semantics that are best suited for the logarithm space and introduce novel simplifications and improvements that are crucial for optimization via gradient-descent.
    We use LTN as the framework for our experiments, but the conclusions of our work apply to any similar NeSy framework. 
    Our findings, both formal and empirical, show that the proposed configuration outperforms the state-of-the-art and that each of our modifications is essential in achieving these results.
\end{abstract}



\section{Introduction}
Recently, there has been an increasing interest in combining logic and neural networks in Neuro-Symbolic (NeSy) integrations. 
The goal of such systems is often to guide the learning of neural networks using symbolic knowledge, allowing them to reason at a higher level of abstraction. 
Much of the recent progress in this area has focused on developing differentiable approaches for knowledge representation and reasoning. 

A trend of approaches involves grounding logical symbols using neural networks and relaxing logical operators into continuous operations using fuzzy semantics.
The resulting formulas, such as $\forall x \exists y P(x,y) \lor R(x)$, are associated with a truth degree in the interval $[0,1]$ that represents their level of satisfiability. 
This satisfiability can then be derived with respect to the parameters of the neural networks that ground the symbols, 
and is incorporated in the loss function of said neural networks to act as an additional supervision when training. 
In this study, we conduct experiments and analyses using Logic Tensor Networks (LTN), a well-established framework for differentiable fuzzy logics. 

Previous research has highlighted that not all fuzzy operators are appropriate for this type of application. 
Different configurations of operators have been proposed in the literature, each with varying degrees of effectiveness, numerical stability, and applicability across different formulas. 
However, as of yet, no configuration has met all of these requirements simultaneously. 
The goal of this paper is to develop a configuration of operators that is superior to previous proposals and capable of handling any formula. 
To achieve this, we propose operators in the logarithm space, which is known to address certain issues.
We build upon existing findings and introduce novel improvements that are crucial for optimization through gradient-descent. 
We call our new solution \logLTN, and release it on the official github repository for LTN.
\footnote{\url{https://github.com/logictensornetworks/logictensornetworks}}

The foundation of the LTN framework is explained in Section \ref{s:ltn_background}.
In Section \ref{s:logltn}, we provide an in-depth examination of the semantics in the logarithm space, with limitations in section \ref{s:comparisons}. 
Our main contribution is in Section \ref{s:logltn_opti}, which includes all the key simplifications and computational techniques, along with their formal justifications, that improve the derivability of the framework.
In Section \ref{s:experiments}, we experimentally confirm that our proposition surpasses state-of-the-art configurations and, using ablation studies, that each of our modifications plays a critical role in achieving these results.
Our findings are expected to help all differentiable frameworks that rely on fuzzy semantics.

\section{Background on LTN}
\label{s:ltn_background}

\subsection{Real Logic concepts}

\def\Friend{\mathsf{is\_friend}}
\def\shortFriend{\mathsf{f}}
\def\bob{\mathsf{b}}
\def\alice{\mathsf{a}}
\def\charlie{\mathsf{charlie}}
\def\shortbob{\mathsf{b}}
\def\shortalice{\mathsf{a}}
\def\shortcharlie{\mathsf{c}}
\def\Italian{\mathsf{Italian}}
\def\height{\mathsf{height}}
\def\istall{\mathsf{is\_tall}}
\def\people{\mathsf{people}}
\def\bx{\mathbf{x}}
\def\by{\mathbf{y}}
\def\bestfriend{\mathsf{best\_friend}}
\def\Parent{\mathsf{is\_parent}}
\def\age{\mathsf{age}}

\label{s:ltn_semantics}
LTN is built on Real Logic, a first-order language that allows to specify relational knowledge about the world.
For example, the formula $\Friend(\alice,\bob)$ states that $\alice$ is a friend of $\bob$,
and the formula $\forall u \forall v (\Friend(u,v) \imp \Friend(v,u))$ states that $\Friend$ is a symmetric relation, 
where $u$ and $v$ are variables, $\alice$ and $\bob$ are individuals, and $\Friend$ is a predicate.

In Real Logic, a grounding $\G_\theta$ associates mathematical, real-valued semantics to every logical symbol depending on a set of parameters $\btheta$. 
Individuals are grounded with vectors of real values.
Often, the vectors come from real-world features and data.
A variable is grounded with a finite batch of individuals from a domain.
Finally, relations are grounded using mathematical functions (generally, neural networks) that map to the truth domain $[0,1]$. 

Complex formulas are constructed using the usual logical connectives and quantifiers $\land$, $\lor$, $\imp$, $\lnot$, $\forall$, $\exists$.
The connectives are grounded using t-norms fuzzy logic:
$\land$ is grounded using a t-norm $\tnorm$, 
$\lor$ using a t-conorm $\snorm$, 
$\imp$ using a fuzzy implication $\fzimp$, 
and $\lnot$ using a negation $\fzneg$.
The quantifiers are grounded using aggregators $\Agg^\forall$ and $\Agg^\exists$.

\begin{example}
    \label{ex:tnorm}
    Examples of fuzzy operators are
    the standard negation $\Ns(x)=1-x$, 
    the product t-norm $\Tprod(x,y)=xy$, 
    and its dual t-conorm $\Sprod(x,y)=x+y-xy$.
    $\Ns$ is inspired by the negation of a probability.
    $\Tprod$ is inspired by the intersection probability of two independent events.
    $\Sprod$ is the dual t-conorm derived from the other two operators using De Morgan's laws.
    For brevity, let us denote $\mathbf{x} = (x_1, \dots, x_n)$ as a vector of $n$ values.
    An example of a universal aggregator is $\Atprod(\mathbf{x}) = \prod_{i=1}^{n} x_i$, which is equivalent to the conjunction of $n$ events.
    Notice that all the operators function within the usual interval $[0,1]$.
\end{example}

\label{s:ltn_optimization}

In LTN, the parameters $\btheta$ are learned using \emph{maximal satisfiability} of a knowledgebase $\K$.
The satisfaction of a formula $\phi$ is its evaluation $\Gt(\phi)$, which returns a truth-value in $[0,1]$.
Let $\K$ define a collection of formulas.
The satisfaction of $\K$ is defined as the aggregation of the satisfactions of each $\phi \in \K$.
The result depends on the choice of aggregate operator, denoted by $\Sat$ (typically, the same operator as the universal aggregator).

The optimal set of parameters maximizes the objective function $  \btheta^\ast = \argmax_{\btheta}\ \Sat_{\phi\in\K}\Gt(\phi)$.
The following loss function is used to find that objective via gradient descent:
\begin{equation}
    \mathcal{L}_{\LTN}(\Gt, \K) = - \Sat_{\phi \in \K} \Gt(\phi)
\end{equation}
For more intuition, we give a concrete example in Appendix \ref{appendix:example_sat_ltn}.

\subsection{Appropriate Operators for Gradient Optimization}
\label{s:gradients}
The ability to find an optimum satisfying a formula greatly depends on the choice of operators that ground the logical connectives.
\cite{van_krieken_analyzing_2022} demonstrate that some fuzzy logic operators are unsuitable in a differentiable setting.
For example, the \L{}ukasiewicz t-norm $\Tluk(x,y) = \max(x+y-1,0)$ has vanishing gradients when $x+y-1<0$.

The authors show that the \emph{Product Real Logic} configuration is the most suitable for grounding the logical connectives.
It uses the product t-norm, its dual t-conorm and the standard negation.
For the universal aggregator, it avoids the potential underflow issues of multiplying many small numbers together by working with the log-product $(\log \circ \Atprod)(\mathbf{x}) = \sum_{i=1}^{n} \log(x_i)$.
Because this configuration mixes operators in the usual and logarithm spaces, it has limitations in expressivity and cannot handle certain formulas (e.g. $(\forall u P(u)) \lor (\forall v Q(v))$).
We discuss this further in Appendix \ref{appendix:prodrl_and_pnf}.

In this paper, we aim to explore a configuration that  
1) performs better than Product Real Logic, 
2) can handle any formula,
and 3) is numerically stable.

\section{Introducing \logLTN}
\label{s:logltn}
We present \logLTN, a specification of LTN with end-to-end semantics in the logarithm space.
Section \ref{s:logltn_semantics} introduces operators that can manipulate appropriately log truth degrees.
Section \ref{s:logltn_opti} shows how to modify these operators to perform well in a differentiable setting.

\subsection{Semantics}
\label{s:logltn_semantics}
We employ the product t-norm $\Tprod(x,y) = xy$ and the maximum t-conorm $\Smax(x,y) = \max(x,y)$.
These operators are known to simplify easily in the logarithm space and are commonly used in the log probability literature.
We also use the standard negation operator $\Ns(1-x)$.
Implications are replaced using the material implication rule $(\phi \imp \psi) \equiv (\lnot\phi \lor \psi)$,
 which means we rewrite every implication using $\fzimp(x,y) = \snorm(\fzneg(x), y)$.

The universal aggregator is defined as the conjunction of $n$ events $\Atprod(\mathbf{x}) = \prod_{i=1}^{n} x_i$,
 and the existential aggregator is defined as the disjunction of $n$ events $\Asmax(\mathbf{x}) = \max_{i=1}^n (x_i)$.

\subsubsection{Logarithm space}
We denote $(\log\circ\Gt)(\phi)$ as the log-grounding of a formula and $(\log \circ\ \mathrm{FuzzyOp})$ as the log-grounding of an operator.
Note that maximizing the log grounding of a formula is equivalent to maximizing its grounding as logarithms are monotone increasing functions.
The log-grounding of $\Tprod$, $\Smax$, and their generalizations in aggregators, simplify easily.
\begin{align}
\label{eq:log_semantics_conj}    (\log\circ \Tprod) (x,y) & = \log(x) + \log(y) \\
\label{eq:log_semantics_disj}    (\log\circ \Smax) (x,y) & = \max(\log(x),\log(y))\\
\label{eq:log_semantics_forall} (\log \circ \Atprod)(\mathbf{x}) & = \sum_{i=1}^n \log(x_i)\\
\label{eq:log_semantics_exst}(\log \circ \Asmax)(\mathbf{x}) & = \max_{i=1}^n (\log(x_i))
\end{align}


Expressing $\log\circ\Ns(x)$ as a function with a logarithmic input 
requires the computation of an exponent and a logarithm.
This means that the operator cannot easily take an input in the logarithm space, 
for example, in $\lnot(A \land B)$.
To overcome this, we write formulas in \emph{negative normal form} (NNF).
A formula is in NNF 
when the scope of each negation operator only applies to atoms (predicates), not to complex formulas, 
and when the formula does not contain any implication or equivalence symbols.
For example, if $A$ and $B$ are two atoms, $\lnot A \land \lnot B$ is in NNF but $\lnot(A \lor B)$ is not.


\subsection{Optimizing in \logLTN}
\label{s:logltn_opti}

\subsubsection{Numerical stability of log negations}
\label{s:lognotsigmoid}
Let $f(x)$ be the output of a neural predicate in the interval $[0,1]$ depending on a mathematical variable $x$.
Converting a value to the logarithm space is a risky operation in a computational graph, as both $\log(f(x))$ and $\frac{\partial \log(f(x))}{\partial x}= \frac{1}{f(x)} \frac{\partial f(x)}{\partial x}$ 
can cause overflow errors when $f(x)$ tends to $0$.

In NeSy AI, predicates are typically grounded using a final sigmoid or softmax layer to normalize outputs in $[0,1]$.
Fortunately, the computation and differentiation of the logarithm of a sigmoid or softmax simplifies to a stable expression (refer to Appendix \ref{appendix:logsigmoid}).
For this reason, most frameworks for automatic differentiation, such as TensorFlow or PyTorch, 
offer built-in and all-in-one layer implementations of the log sigmoid and log softmax functions.
These should be used when log-grounding a predicate to avoid unstable gradients.

However, the same issue arises when log-grounding the negation of a predicate, $\log(1-f(x))$, 
Fortunately, we show how to reformulate the log-negation of a sigmoid or softmax predicate to numerically stable expressions.

\begin{restatable}{theorem}{lognotsigmoid}
    \label{thm:log_not_sigmoid}
    The log-negation of a sigmoid function $S(x)=\frac{1}{1+e^{-x}}$, $x \in \R$, simplifies as 
    \begin{equation}
        \label{eq:log_not_sigmoid}
        (\log\circ\Ns)(S(x)) = \log(S(x)) - x
    \end{equation}
\end{restatable}
\begin{proof}
    Proof in Appendix \ref{appendix:proof_log_not_sigmoid}.
\end{proof}

\begin{restatable}{theorem}{lognotsoftmax}
    \label{thm:log_not_softmax}
    The log-negation of a softmax function $\softmax(\mathbf{z})_i = \frac{e^{z_i}}{\sum_{j=1}^{K} e^{z_j}}$,
    where $\mathbf{z} =(z_{1},\dotsc ,z_{K})\in \mathbb {R} ^{K}$ is a vector of $K$ real values,
    and $i=1,\dots,K$, simplifies as
    \begin{equation}
        \label{eq:log_not_softmax}
        (\log \circ \Ns) (\softmax(\mathbf{z})_i) = \log(\softmax(\mathbf{z})_i) + \log (\sum_{\substack{j=1\\j\neq i}}^K e^{z_j}) - z_i
    \end{equation}
\end{restatable}
\begin{proof}
    Proof in Appendix \ref{appendix:proof_log_not_softmax}.
\end{proof}

These two proposed reformulations have numerically stable implementations.
The first uses the logarithm of a sigmoid and a linear term.
The second uses the logarithm of a softmax function, a linear term, and a logarithm of a sum of exponentials, also known as LogSumExp.
LogSumExp also has a stable implementation and its derivative is a softmax function.

Below, we briefly show the stability advantage of our reformulation for sigmoid by comparing its output with a naive definition $\log(1-\sigmoid(x))$.
The results are obtained in TensorFlow with float32 precision.
The same can be reproduced with the softmax reformulation. 

\begin{verbatim}
Input :
    x : [0., 10., 100., 1000., 10000.]
Output :
    f1(x)=log(1-S(x)) : [-0.69, -1.0e+1, -inf, -inf, -inf]
    df1/dx(x)         : [-0.5, -1.0, nan, nan, nan]
    f2(x)=log(S(x))-x : [-0.69, -1.0e+1, -1.0e+2, -1.0e+3, -1.0e+4]
    df2/dx(x)         : [-0.5, -1.0, -1.0, -1.0, -1.0]
\end{verbatim}


\subsubsection{Relaxation of the disjunctions}
\label{s:logsumexp}
The maximum operator in equations \eqref{eq:log_semantics_disj} and \eqref{eq:log_semantics_exst} is unsuitable in a differentiable setting as it has single-passing gradients.
This means that it only propagates gradients to one input at a time, the one with the highest value.
Intuitively, let the formula $\exists x \ P(x)$ be a constraint used to optimize a neural predictor $P$.
If several individuals in the batch $x$ tend to verify $P(x)$, $\max$ will have non-zero gradients for only one of them.
This can be inefficient in practice as it will push the predictor to overfit that single individual in $x$ and ignore the others. Also, it is particularly sensitive to initial conditions.

A common solution is to use a smooth approximation of the maximum operator.
A popular candidate in the logarithm space is the LogSumExp ($\LSE$) operator, defined as:
\begin{align}
    \LSE(\mathbf{x} \mid \alpha, C) = \frac{1}{\alpha} \left( C + \log(\sum_{i=1}^n e^{\alpha x_i - C})\right)\\
\end{align}
$C=\max(\alpha \mathbf{x})$ is a constant that does not change the result of the expression but prevents overflow errors in the exponential terms.
$\alpha$ is a hyperparameter that scales the bounds of $\LSE$ according to the following inequality:
\begin{equation}
    \label{eq:bounds_LSE_traditional}
    \max(\mathbf{x}) \leq \LSE(\mathbf{x} \mid \alpha, C) \leq \max(\mathbf{x}) + \frac{\log(n)}{\alpha}
\end{equation}
Here, we identify an issue in that LogSumExp approaches the maximum value via a higher bound.
This is problematic, as truth degrees are bound to the interval $[0,1]$, 
and log truth degrees should be bound in the interval $\left[-\infty,0\right]$.
\footnote{The edge case $\log(0) = -\infty$ can be avoided by add a small real value $\epsilon>0$ to zero truth degrees. 
However, this is rarely a problem in practice as sigmoid and softmax layers output values in $]0,1[$.}
With $\LSE$, the output of a log-disjunction can exceed these bounds and become non-negative.

To address this issue, we propose the use of a LogMeanExp operator $\LME$:
\begin{equation}
    \label{eq:LME}
    \LME(\mathbf{x} \mid \alpha, C) = \frac{1}{\alpha} \left( C + \log(\frac{\sum_{i=1}^n e^{\alpha x_i-C}}{n} ) \right)
\end{equation}
It approaches the maximum operator from below values (proof in Appendix \ref{appendix:bounds_LSE_lower}):
\begin{equation}
    \label{eq:bounds_LSE_lower}
    \max(\mathbf{x}) - \frac{\log(n)}{\alpha} \leq \LME(\mathbf{x} \mid \alpha, C) \leq \max(\mathbf{x})
\end{equation}

This operator is numerically stable, well-bounded, and suitable for derivation.
We use it to ground disjunctions and existential quantifications in \logLTN. 
For best practice, the parameter $\alpha$ that scales the smooth maximum should be scheduled over time to balance exploration and exploitation~\cite{badreddine_logic_2022}.

\subsubsection{Batch-size invariance for the universal aggregation}
\label{s:batch_aggregation}
We have improved the derivability 
of negations, disjunctions, and existential quantifiers in the logarithm space.
Here, we identify an issue with the universal quantification.

Consider a knowledge base with two rules $\phi_1 = \forall u P(u)$ and $\phi_2 = \exists v Q(v)$.
Let $\G(u) = [x_1, \dots, x_m]$ and $\G(v) = [y_1, \dots, y_n]$ be two batches of individuals.
Let us develop the groundings of the rules:
\begin{align}
\label{eq:batch_inv_forall}     (\log\circ\Gt)(\phi_1) &
         = \sum_{i=1}^m (\log\circ\Gt)(P)(x_i) \\
\label{eq:batch_inv_exists}     (\log\circ\Gt)(\phi_2) &
         = \max_{i=1}^n (\log\circ\Gt)(Q)(y_i) 
\end{align}

In the loss $\mathcal{L} = - (\log\circ\Gt)(\phi_1 \land \phi_2) = - (\log\circ\Gt)(\phi_1) - (\log\circ\Gt)(\phi_2)$,
the optimization will tend to overfit the rule with the universal quantifier and ignore the existential rule for large batch sizes.
This is due to the fact that the first sums $m$ log truth degrees, whereas the second only takes one log truth degree as a maximum.
In terms of differentiability,
the universal rule weights more on the gradient updates, 
as $\sum_{i=1}^m \frac{\partial \mathcal{L}}{\partial \log\circ\G(P)(x_i)} = \sum_{i=1}^m -1 = -m$,
whereas $\sum_{i=1}^n\partialderivative{\mathcal{L}}{\log\circ\G(Q)(y_i)} = \sum_{i=1}^n - \delta_{ij} = -1$ given $j=\argmax_{j=1}^n \G(Q)(y_j)$.
The problem remains with the smooth maximum LogMeanExp, whose gradients are a softmax function summing to 1 as well.

This weighing problem can also arise when comparing two universal quantifiers, such as $\forall u P(u)$ and $\forall v Q(v)$. 
If the batches for $u$ and $v$ have varying sizes, the optimization algorithm will tend to overfit the rule that has more examples of individuals and ignore the other.

To solve this problem, we propose to use a mean instead of a sum as a weighting scheme to balance universal quantifiers:
\begin{equation}
\label{eq:forall_invariant}
    (\log\circ\G)(\forall u P(u)) = \sum_{i=1}^m \frac{(\log\circ\G)(P)(x_i)}{m}
\end{equation}
By averaging log truth degrees instead of summing them, we obtain a \emph{batch-size invariant} aggregator.
The weight of the gradients becomes $\sum_{i=1}^m \frac{\partial \mathcal{L}}{\partial \log\circ\G(P)(x_i)} = -1$.
This ensures every formula weighs equally in the loss function.

In the normal space, the universal quantifiers then correspond to geometric means instead of products.
This trick alters the objective and search space of the task, but we find that it is crucial to approach good solutions in our experiments.

\subsubsection{Summary}
By implementing all the aforementioned modifications, 
we arrive at the log-grounding routine presented in Algorithm \ref{alg:log_grounding_derivable}.

\begin{algorithm}
\caption{Compute $(\log\circ\Gt)(\phi)$ for derivability}
\label{alg:log_grounding_derivable}
\begin{algorithmic}[1]
    \State {\bf Step 1} \ {\it Rewrite $\phi$ in negative normal form}
    \State {\bf Step 2} \ {\it Log-ground predicates and their negations}
        \State \indent Use the log-negation simplifications with sigmoid/softmax layers (Equations \ref{eq:log_not_sigmoid} and \ref{eq:log_not_softmax})
    \State {\bf Step 3} \ {\it Compute connectives in the logarithm space}
        \State \indent $\land$ becomes $+$
        \State \indent $\lor$ becomes $\LME$
        \State \indent $\forall$ becomes $\mean$
        \State \indent $\exists$ becomes $\LME$        
\end{algorithmic}
\end{algorithm}

\section{Limitations}
\label{s:comparisons}
\subsection{De Morgan's Laws and NNF}
\label{s:NNF}
The Negative Normal Form (NNF) does not preserve equivalence with the logarithmic semantics introduced in Section \ref{s:logltn_semantics}.
To transform a formula in NNF, one must push the negations in front of atoms using De Morgan's laws, but the laws do not hold given that $\Tprod(x,y)=xy$ and $\Smax(x,y) = \max(x,y)$ are not fuzzy dual operators.
However, we can prove the following De Morgan's inequalities:

\begin{restatable}{theorem}{logdemorgans}
    \label{thm:log_demorgans}
    Let $P$ and $Q$ be two formulas. We can show that
    \begin{align}
        \G(\lnot(P \land Q)) & \geq \G(\lnot P \lor \lnot Q) \\
        \G(\lnot(P \lor Q)) & \geq \G(\lnot P \land \lnot Q) \\
        \G(\lnot(\forall u P(u))) & \geq \G(\exists u \lnot P(u)) \\
        \G(\lnot(\exists u P(u))) & \geq \G(\forall u \lnot P(u))
    \end{align}
\end{restatable}
\begin{proof}
    Proof in Appendix \ref{appendix:proof_nnfdemorgans}.
\end{proof}
These results stem from the fact that $\max(x,y)$ is a lower-bound to other t-conorms including the dual product t-conorm.
In Appendix \ref{appendix:bound_tightness}, we analyze the tightness of these bounds.
Given that NNF is obtained by repeatedly applying De Morgan's laws, we can infer the following property:
\begin{corollary}
    Let $\phi$ be any formula and $\phi'$ be a NNF formula derived syntactically from $\phi$ using De Morgan's laws.
    Then, $\Gt(\phi') \leq \Gt(\phi)$, meaning that the satisfaction of the NNF formula $\phi'$ is a lower bound for the satisfaction of the formula $\phi$.
\end{corollary}
This is particularly useful as, if we convert a formula into NNF and find a parametric grounding that satisfies it, we know that the original formula is at least as satisfied.


\section{Experiments}
\label{s:experiments}
\def\StableRL{LTN-Stable\ }
\def\ProdRL{LTN-Prod\ }
\def\logLTNsum{\logLTN-sum\ }
\def\logLTNmax{\logLTN-max\ }
\def\logLTNLSE{\logLTN-LSE\ }

\subsection{Task 1: Clustering}
\label{s:task_clustering}
The first experiment is a clustering problem based on the gene expression cancer RNA-Seq benchmark from the UCI ML datasets repository~\cite{UCI_Dua_2019}. 
The dataset has 801 samples of 20531 features, which we reduce to 16 features using PCA. 
The task is to divide samples into five clusters, which roughly correspond with five ground truth cancer types. 
We train a neural predictor $C(x,c)$ that returns the belief of a point $x$ belonging in a given cluster $c$ using these three constraints:
\begin{gather}
\label{eq:clustering1}    \forall x \exists c\ C(x,c)\\
\label{eq:clustering2}    \forall c \exists x\ C(x,c)\\
\label{eq:clustering3}    \forall (c,x,y: \lvert x-y \rvert < \mathrm{th}_{2.5})\ C(x,c) \imp C(y,c) 
\end{gather} 
$x$ and $y$ are grounded with the batch of 801 points.
$c$ is a variable that ranges over five cluster ids.
$C(x,c)$ outputs beliefs using a softmax output layer that ensures mutual exclusivity of clusters.
\eqref{eq:clustering3} uses the concept of "guarded quantification" introduced by \cite{badreddine_logic_2022}.
It means that the quantification only retains the individuals verifying the condition $\lvert x-y \rvert < \mathrm{th}_{2.5}$,
where $\mathrm{th}_{2.5}$ is the $2.5$-th percentile of the euclidean distances between all pairs of points.
Intuitively, the constraint states that for any pair of points that are very close, if one belongs to a cluster, the other must belong in the same cluster.

We use this task, inspired by the toy example from \cite{badreddine_logic_2022} and extended on real-world data, as it is one of the rare NeSy tasks with existential clauses ranging over many individuals (here, 801 individuals). 

\subsection{Task 2: MNISTAdd}
\def\mnistthree{\inlinegraphics{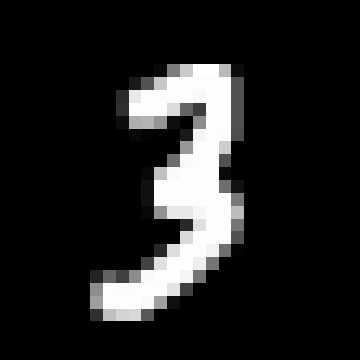}}
\def\mnisteight{\inlinegraphics{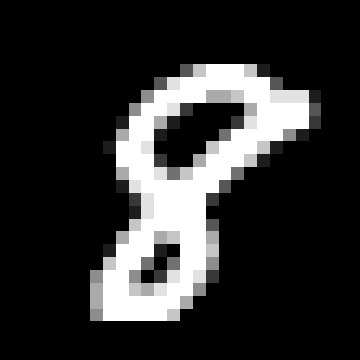}}
\def\mnistnine{\inlinegraphics{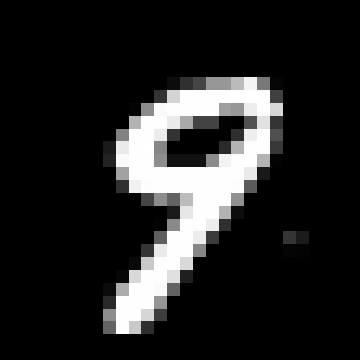}}
\def\mnisttwo{\inlinegraphics{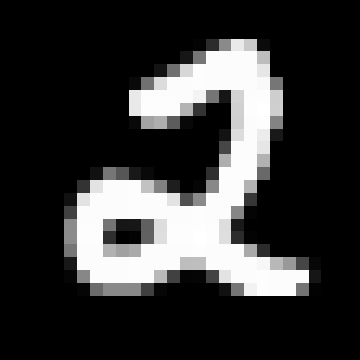}}
\def\isdigit{\mathsf{is\_digit}}

The second problem is the MNISTAdd task from \cite{manhaeve_deepproblog_2018}.
We experiment on the 2-digits number variant of the task.
In this problem, we learn to recognize the sum of two numbers of two digits using only the result of the sum as a training label.
For example, a dataset sample would be $([\mnistthree,\mnisteight], [\mnistnine,\mnisttwo], 130)$.

The task is modeled using a digit classifier $\isdigit(\mnisteight,d)$ which predicts beliefs for the MNIST image being the digit $d=0\dots 9$.
Whereas we only provide labels for the final result of the addition, NeSy methods use prior knowledge about intermediate labels (possible digits used in the addition) to propagate ground truth information to the digit classifier.
Given the MNIST images and sum result $([x_1,x_2],[y_1, y_2],n)$, we use the LTN constraint from \cite{badreddine_logic_2022}:
\begin{align}
& \exists d_1,d_2,d_3,d_4 : 10d_1+d_2+10d_3+d_4=n \\
& \quad (\isdigit(x_1,d_1) \land \isdigit(x_2,d_2) \land \isdigit(y_1,d_3) \land \isdigit(y_2,d_4)) \notag
\end{align}
The loss signal is a universal aggregation of the constraint over minibatches of labeled examples. 
We use the same neural network for $\isdigit$ as \cite{manhaeve_deepproblog_2018}. 
This is a basic experiment with a single training constraint. 
We assess it because many probabilistic NeSy methods use it as a standard for comparison.

\subsection{Task 3: Semantic PASCAL-Part}
\label{s:task_pascalpart}
\newcommand{\isbottle}{\mathsf{bottle}}
\newcommand{\iscap}{\mathsf{cap}}
\newcommand{\istype}{\mathsf{is}}
\newcommand{\bottle}{\mathsf{bottle}}
\newcommand{\captype}{\mathsf{cap}}
\newcommand{\body}{\mathsf{body}}
\newcommand{\ispartof}{\mathsf{partOf}}
\newcommand{\isbody}{\mathsf{body}}

The third experiment is a semi-supervised semantic image interpretation task on the semantic PASCAL-Part dataset from \cite{donadello_ltnimage_2017}.
The goal is to train a type classifier $\istype(x,\bottle)$, $\istype(x,\captype)$, etc., that predicts the type of an object within a bounding box $x$,
 and to train a relation predictor $\ispartof(x,y)$ that determines if one bounding box $x$ is part of another bounding box $y$.
 An example of such bounding boxes is presented in Figure \ref{fig:pascal_example}.
 
 \begin{figure}
    \centering
    \includegraphics[width=0.4\linewidth]{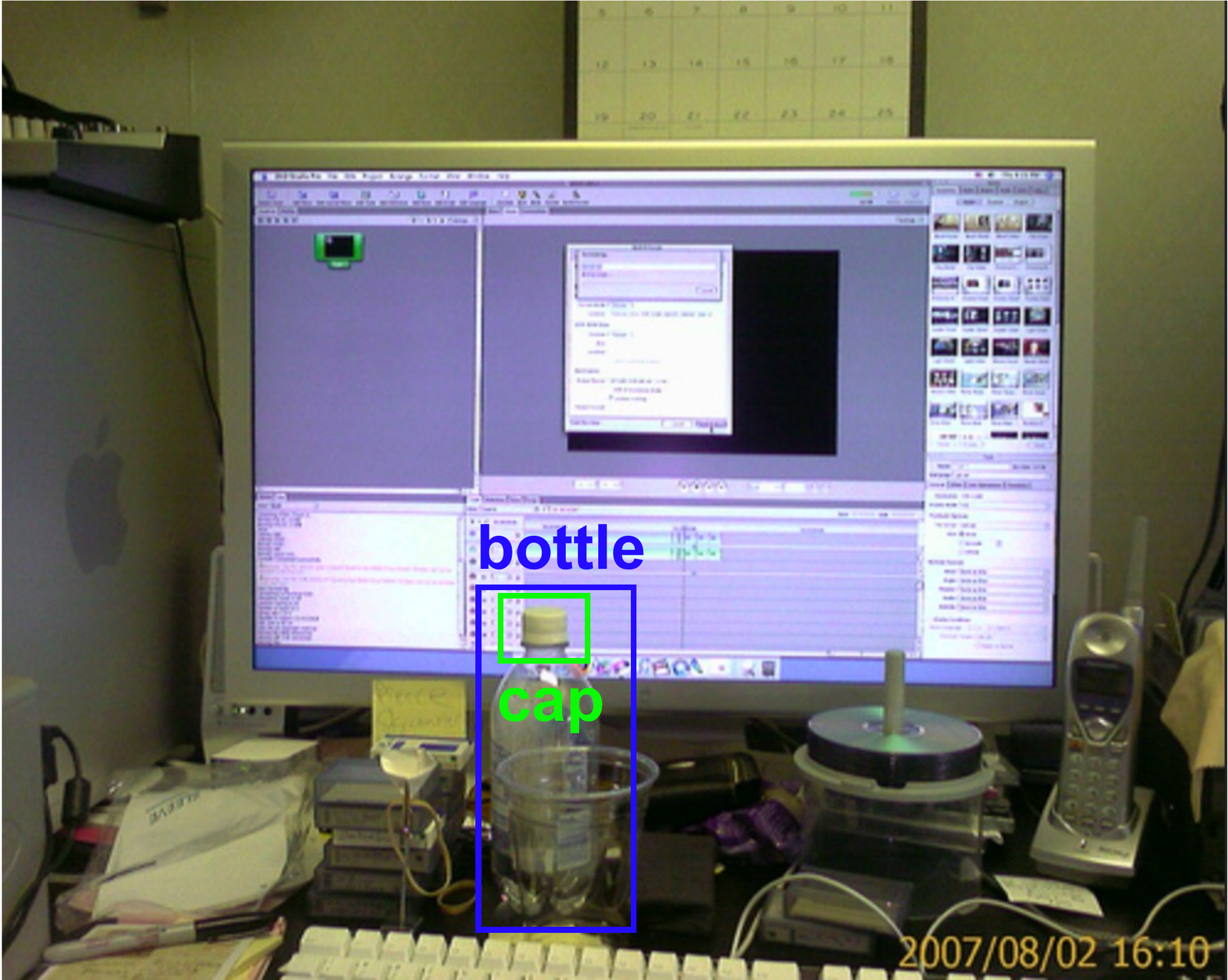}
    \caption{Example of bounding boxes from PASCAL-Part.}
    \label{fig:pascal_example}
\end{figure}

Training is guided by three constraints based on ground truth examples,
one for labeled type examples, one for pairs of positive examples for $\ispartof$, and one for pairs of negative examples. 
The ground truth labels are made available for only $5\%$ of the training data.
However, training is carried on the unlabeled data using mereological constraints that relate to the types and their meanings, for example:
\begin{gather}
    \forall x,y \ \istype(x,\bottle) \land \ispartof(y,x) \imp (\istype(y,\captype) \lor \istype(x,\body)) \label{eq:constraint_pascalpart} \\
    \forall x,y \ \istype(x,\captype) \land \ispartof(x,y) \imp \istype(y,\bottle)
\end{gather}

\cite{donadello_ltnimage_2017} grounded the bounding boxes using predictions produced by an object detector trained on PASCAL-Part.
This means that LTN was only used to correct the predictions of the detector.
We increase the difficulty of the task by implementing the bounding boxes with a latent vector of 1024 features output by a pre-trained FasterRCNN backbone.
That is, LTN has to learn all the final layers of the object detector and its specialization on PASCAL-Part.
We release our version of the dataset on \url{https://github.com/sbadredd/semantic-pascal-part}.
More details on the experiment are available in Appendix \ref{appendix:pascalpart}.

Out of the three tasks, this is by far the largest with a total 59 object types and 60 corresponding constraints.
It showcases the power of LTN and its capability to simply integrate many constraints in a loss function. 
We evaluate the type classification using balanced accuracy, $\ispartof(x,y)$ using the area under precision-recall curves, 
and the semantic interpretation by reporting the number of false positives that violate the mereological constraints -- for example, a bottle is predicted to be part of a cap.


\subsection{Baselines}
We compare \logLTN\ with the following baselines. 
Note that we skip the ablation of the log-negation simplifications, as their numerical practicability is already illustrated in Section \ref{s:lognotsigmoid}.
\begin{description}
    \item[\ProdRL] Product Real Logic was identified by \cite{van_krieken_analyzing_2022} as the best performing operator semantics for differentiable fuzzy logics.
    It uses the product t-norm $\Tprod(x,y)=xy$ and its dual t-conorm $\Sprod(x,y)=x+y-xy$.
    The universal quantifier uses the log-product aggregator $(\log \circ \Atprod)(\mathbf{x}) = \sum_{i=1}^{n} \log(x_i)$
    and the existential quantifier uses a smooth maximum.
    Its combination of operators both in the usual and logarithm space makes it difficult to handle certain formulas. 
    We discuss this issue further in appendix \ref{appendix:prodrl_and_pnf}.
    \item[\StableRL] Stable Product Real Logic~\cite{badreddine_logic_2022} is a modification of \ProdRL that uses a smooth minimum for the universal aggregator,
    such that all operators perform in the usual space.
    A limitation of the smooth minimum is that it depends on a smoothing hyperparameter $p$, which we show to greatly influence the results.
    \item[\logLTN] the configuration introduced in this paper, performing fully in the logarithm space.
    \item[\logLTNsum] an ablation of logLTN using a sum instead of a batch-size invariant mean for universal aggregations. See Section \ref{s:batch_aggregation}
    \item[\logLTNmax]  an ablation logLTN that uses a non-relaxed maximum operator for existential aggregations. See Section \ref{s:logsumexp}.
    \item[\logLTNLSE] an ablation of logLTN that uses a traditional LogSumExp operator for existential aggregations instead of LogMeanExp. See Section \ref{s:logsumexp}.
\end{description}

\section{Results}
The code for our experiments is available at \url{https://github.com/sbadredd/logltn-experiments}.
We perform experiments with runs of 1000 training steps for the clustering problem, runs of 20 epochs for MNISTAdd with two different dataset sizes, and runs of 1000 training steps for Semantic PASCAL-Part.
The results are summarized in Tables \ref{tab:results_cluster_mnistadd} and \ref{tab:results_pascalpart}. 
Because Semantic PASCAL-Part is more computationally demanding, we conducted ablation studies only on the two first experiments.
For additional implementation details, please refer to Appendix \ref{appendix:implementation} covering training configuration and baseline hyperparameters.
Across all metrics, \logLTN\ consistently achieved the best or second best performance.

\ProdRL performs poorly in the clustering task.
In a qualitative analysis of the cluster assignments (Figure \ref{fig:results_clustering}), we observe that \ProdRL disregards constraint \eqref{eq:clustering2} stating that each cluster contains at least one point.
This is due to the batch-variant log-product. 
As that constraint is quantified over five clusters ($\forall c$), it has relatively less weight compared to the other ones aggregated over all points.
Also, in Semantic PASCAL-Part, while \ProdRL exhibits good results, \logLTN\ still demonstrates superior performance by avoiding on average 34\% more mereological violations and having less deviation across all metrics than \ProdRL.

We tested \StableRL with the smooth minimum parameter $p=2$ and $p=6$.
When $p=2$, the smooth minimum is less strict and corresponds to a Mean-Squared Error (MSE) aggregator. 
This leads to low accuracy in the PASCAL-Part problem, as the constraint aggregator focuses on satisfying the 57 logical constraints rather than the three ground truth constraints which it treats as "outliers".
The result is a predictor classifying all objects into barely constrained types (e.g. background) and all $\ispartof(x,y)$ as false negatives in order to reach low mereological violations.
With $p=6$, the aggregator is more strict but can overfit outliers and exhibit instability in other experiments.
Despite its name, we find \StableRL to be too dependent on the hyperparameter $p$ and unstable.

Regarding the ablations, \logLTNmax generally performed poorly due to inadequate gradient propagation. 
\logLTNLSE showed similar performance to \logLTN\ overall, except for deviating results in MNISTAdd, possibly due to the unbounded maximum breaking at an edge case. 
Also, even on well-performing problems, we find that grounding the knowledgebase with $\logLTNLSE$ can lead to log truth degrees reaching values as high as $(\log\circ\G)(\K)=5$.
That corresponds to a truth degree of approximately $(\log\circ\G)(\K)=150$. 
Since fuzzy truth degrees should be within the range of $[0,1]$, the higher values generated by \logLTNLSE make it unusable in many cases, making \logLTN\ a more suitable option.

Finally, \logLTNsum exhibits behavior akin to \ProdRL in clustering due to its batch-variant aggregator.
In MNISTAdd, \logLTNsum outperforms \logLTN, but the only difference between the two baselines is a constant factor in the loss function due to taking a mean over the minibatch of samples instead of a sum.
We assume that scaling the learning rate accordingly would yield comparable results with \logLTN.

In Table \ref{tab:reported_results_mnistadd}, we compare our MNISTAdd results with those reported by popular probabilistic frameworks~\cite{manhaeve_deepproblog_2018,winters_deepstochlog_2022,pryor2022neupsl}. 
In their study, \cite{badreddine_logic_2022} showed that LTN managed to solve the MNISTAdd problem, but the outcomes varied significantly due to instability during initialization. 
We show that by training with \logLTN\ in the logarithm space, we resolved this issue and achieved standard state-of-the-art performance in the task.

\begin{figure}
    \centering
    \includegraphics[width=\textwidth]{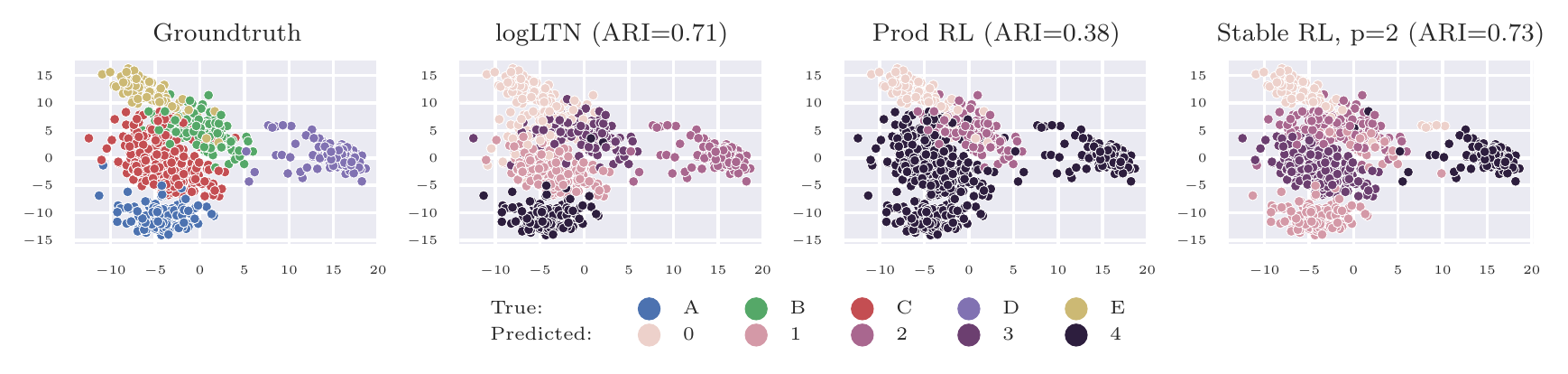}
    \caption{2D PCA plots of typical cluster assignments by each baseline}
    \label{fig:results_clustering}
\end{figure}

\begin{table}
    \centering
    \begin{tabular}{l|c|cc}
        \toprule
        & {\it Clustering} & \multicolumn{2}{c}{{\it MNISTAdd}} \\         
         &  &  1,500 samples & 15,000 samples \\
        \midrule
        \ProdRL & $0.37 \pm 0.18$ & $ 88.39 \pm 0.97 ^*$ & $ 95.39 \pm 0.33 $ \\
        \StableRL (p=2) & $0.71 \pm 0.10$ & $ 62.77 \pm 37.16 $ & $ 80.34 \pm 32.82 $ \\
        \StableRL (p=6) & $0.62 \pm 0.05$ & $ 30.51 \pm 34.19 $ & $ 95.12 \pm 0.62 $ \\
        \logLTN & $0.71 \pm 0.06 ^*$ & $ 88.29 \pm 0.78 $ & $ 95.61 \pm 0.51 ^* $ \\
        \logLTNLSE & $ {\bf 0.72\pm 0.05 } $  & $ 75.35 \pm 29.96 $ & $ 95.27 \pm 0.53 $ \\
        \logLTNmax & $0.35\pm 0.17$ & $ 48.39 \pm 36.80 $ & $ 80.81 \pm 32.64 $ \\
        \logLTNsum & $0.36\pm 0.22$ & $ {\bf 88.49 \pm 0.88}  $ & $ {\bf 95.62 \pm 0.42}  $ \\
        \bottomrule
    \end{tabular}
    \caption{Results on Clustering (Adjusted Rand Index scores -- averaged on 10 runs) and MNISTAdd (test accuracy -- averaged on 5 runs).}
    \label{tab:results_cluster_mnistadd}
\end{table}

\begin{table}
    \centering
    \begin{tabular}{lcc}
        \toprule
         & 1,500 samples & 15,000 samples \\
        \midrule
        DeepProbLog & $ 87.21 \pm 1.92^* $ & $ 95.16 \pm 1.70 $ \\
        DeepStochLog & NA & $ {\bf 96.40 \pm 0.10} $ \\
        NeuPSL & $ 87.05 \pm 1.48 $ & $ 93.91 \pm 0.37 $ \\
        \logLTN & $ {\bf 88.29 \pm 0.78} $ & $ 95.61 \pm 0.51^* $ \\
        \bottomrule
    \end{tabular}
    \caption{
            Reported test accuracy on MNISTAdd by probabilistic baselines. 
            DeepStochLog does not report results on the training set size 1,500.}
    \label{tab:reported_results_mnistadd}
\end{table}

\begin{table}
    \centering
    \begin{tabular}{lccc}
        \toprule
         & PartOf AUC & Type Accuracy & \# Mereological Violations \\
        \midrule
        \ProdRL & ${\bf 71.38 \pm 10.38} $ & $53.40 \pm 1.81^*$ & $16,312.8 \pm 4,180.2 $ \\
        \StableRL (p=2) & $18.62 \pm 28.16$ & $1.90 \pm 0.46$ & ${\bf 1,196.0 \pm 2,674.3}$ \\
        \StableRL (p=6) & $64.85 \pm 3.05$ & $34.53 \pm 1.62$ & $24,324.5 \pm 14,052.3$ \\
        \logLTN & $69.57 \pm 4.33 ^*$ & ${\bf 55.54 \pm 1.00}$ & $10,708.2 \pm 1,455.9 ^*$ \\
        \bottomrule
    \end{tabular}
    \caption{Test results on Semantic PASCAL-Part averaged on 5 runs.}
    \label{tab:results_pascalpart}
\end{table}

\section{Related Work}
The field of combining logic and neural networks in NeSy integrations is gaining interest, as outlined by \cite{garcez_neurosymbolic_2020}. For an overview of the approaches and challenges, see \cite{hitzler_nesysota_2022}. To understand the prevalence of these systems, refer to \cite{mdsarker_nesy_insights_2021}.

A family of approaches converts logical connectives into differentiable operations using fuzzy semantics. 
Systems that employ this approach include 
LTN \cite{serafini_ltn_2016,badreddine_logic_2022}, 
KALE \cite{guo_kale_2016}, 
SBR \cite{diligenti_integrating_2017},
and LRNN \cite{sourek_lifted_2018}
among others.
Unlike probabilistic logics~\cite{manhaeve_deepproblog_2018,winters_deepstochlog_2022}, fuzzy approaches change logic semantics and are less common in proof reasoning.
Nevertheless, fuzzy frameworks excel in knowledge-aided learning and offer simplicity compared to probabilistic methods which must often solve the exponentially complex model counting problem.

Fuzzy frameworks have been used in a wide range of applications and fields in recent years.
These include but are not limited to semantic image interpretation \cite{donadello_ltnimage_2017}, 
natural language processing \cite{bianchi_complementing_2019},
reinforcement learning \cite{injecting_prior_RL_badreddine_2019}, 
query answering over knowledge graphs \cite{arakelyan_cqa_2021,fuzzqe},
or open-world reasoning \cite{wagner_reasoning_2022}.
Our paper aligns with the research stream of \cite{van_krieken_analyzing_2022} as it strives to improve the performance of all these related works by providing mathematical and computational cues for fuzzy semantics.

\section{Conclusions}
Many NeSy approaches rely on fuzzy operator semantics to ground knowledge in loss functions. 
However, it is clear that not all semantics are suitable for gradient descent optimization algorithms. 
In this paper, we propose a set of semantics that can be used to train logic end-to-end in the logarithm space. 
We demonstrate that the proposed configuration outperforms semantics previously considered state-of-the-art in such NeSy systems.

We propose the solution, which we refer to as logLTN, as an additional set of semantics for LTN and implement it in the repository of the framework. 
Each of our findings can also be applied separately to any framework that works with logic in the logarithm space. 
In summary, our recommendations for such systems include computing log-negations using Equations \eqref{eq:log_not_sigmoid} and \eqref{eq:log_not_softmax}, relaxing disjunctions using Equation \eqref{eq:LME}, and making universal quantifications batch size-invariant using Equation \eqref{eq:forall_invariant}. 
Our research is expected to improve the performance of all NeSy approaches that rely on fuzzy operator semantics.

\bibliographystyle{authordate1} 
\bibliography{main}

\newpage

\appendix
\section{Background}

\subsection{LTN Example}
\label{appendix:example_sat_ltn}
Let us denote the predicate $\Friend$ as $\shortFriend$ for brevity.
In the expression $\shortFriend(\shortalice,\shortbob)$, let $\Gt(\shortalice)$ and $\Gt(\shortbob)$ be vector embeddings in $\R^m$.
A primitive approximation of the friendship relationship could be a cosine similarity function $\Gt(\shortFriend) : (\bx,\by)\mapsto \frac{\bx \cdot \by}{||\bx|| ||\by||} $.  
If $\Gt(\shortalice) = \begin{bsmallmatrix}3&2&0&5\end{bsmallmatrix}$ and $\Gt(\shortbob) = \begin{bsmallmatrix}1&0&0&4\end{bsmallmatrix}$, we have 
$\Gt(\shortFriend(\shortalice,\shortbob)) = 0.905$; that is, a high truth degree. 

Of course, stating that people are friends if they are similar is primitive.
In a real-case scenario, the friendship relationship would likely be approximated by a parametric function, such as a neural network, 
and trained based on constraints in a loss function.

Consider the formula $\phi = \lnot \shortFriend(\shortalice, \shortbob) \lor \shortFriend(\shortbob, \shortalice)$
and a knowledgebase that contains this unique formula $\K = \{ \phi \}$.
The formula states that if $\alice$ is a friend of $\bob$, then $\bob$ is a friend of $\alice$.
    \footnote{$\phi$ is semantically equivalent to $\shortFriend(\shortalice, \shortbob) \imp \shortFriend(\shortbob, \shortalice)$ 
            if we use a material implication defined as $p \imp q \equiv \lnot p \lor q$.}
Let the grounding for $\Gt(\shortFriend(\shortalice, \shortbob))$ depend 
on a trainable neural network for the friendship relation
and a set of features for $\alice$ and $\bob$.
To update $\btheta$ via gradient descent steps, we calculate 
$\frac{\partial \Gt(\phi)}{\partial \Gt(\shortFriend(\shortalice,\shortbob))}$ 
and 
$\frac{\partial \Gt(\phi)}{\partial \Gt(\shortFriend(\shortbob,\shortalice))}$.
Using the operators $\Ns(x) = 1-x$ and $\Sprod(x,y)=x+y-xy$, we get:
\begin{align}
    \Gt(\phi) & =  \Sprod(\Ns(\Gt(\shortFriend(\shortalice,\shortbob))), \Gt(\shortFriend(\shortbob,\shortalice))) \\
            & = 1 - \Gt(\shortFriend(\shortalice,\shortbob)) + \Gt(\shortFriend(\shortalice,\shortbob)) \Gt(\shortFriend(\shortbob,\shortalice))
\end{align}
And the partial derivatives:
\begin{align}
    \label{eq:partial_friendba} \frac{\partial \Gt(\phi)}{\partial \Gt(\shortFriend(\shortbob,\shortalice))} & = \Gt(\shortFriend(\shortalice,\shortbob)) \\
    \label{eq:partial_friendab} \frac{\partial \Gt(\phi)}{\partial \Gt(\shortFriend(\shortalice,\shortbob))} & = -1+\Gt(\shortFriend(\shortbob,\shortalice))
\end{align}%

Equations \eqref{eq:partial_friendab} and \eqref{eq:partial_friendba} give us interesting insights on the power of LTN.
When maximizing the satisfiability of the formula, if $\shortFriend(\shortalice,\shortbob)$ is high,
then $\frac{\partial \Gt(\phi)}{\partial \Gt(\shortFriend(\shortbob,\shortalice))}$ is high.
Intuitively, if $\alice$ being friend with $\bob$ has a high truth value, LTN will tend to increase the truth value of $\bob$ being friend with $\alice$.
Alternatively, if $\shortFriend(\shortbob,\shortalice)$ is low, $\frac{\partial \Gt(\phi)}{\partial \Gt(\shortFriend(\shortalice,\shortbob))}$ is close to $-1$.
That means that if $\bob$ is not considered friend with $\alice$, LTN will tend to decrease the truth value of $\alice$ being friend with $\bob$.

These are different scenarios and ways in which the framework pushes parametric groundings to verify logical constraints.
One can easily imagine how the LTN loss can be used as an additional loss term when training neural networks or embeddings to find a balanced optimum that also satisfies a knowledgebase.

\subsection{Product Real Logic and Prenex Normal Form}
\label{appendix:prodrl_and_pnf}
The study conducted by \cite{van_krieken_analyzing_2022} evaluated a range of operators for differentiability and found Product Real Logic to be the current state-of-the-art operator semantics for differentiable fuzzy logics.
This set of semantics uses the product t-norm $\Tprod(x,y)=xy$, its dual t-conorm $\Sprod(x,y)=x+y-xy$, the standard negation $\Ns(x)=1-x$, and the material implication.
The universal quantifier uses the log-product aggregator $(\log \circ \Atprod)(x_1,\dotsc,x_n) = \sum_{i=1}^{n} \log(x_i)$,
and the existential quantifier uses a smooth maximum.

A limitation of this set of semantics is that it combines operators in both the standard space and the logarithm space, making it challenging to handle certain formulas. 
For example, a formula such as $(\forall u P(u)) \lor (\forall v Q(v))$ cannot be grounded as is, as the $\forall$ operator outputs a log truth degree while the $\lor$ operator expects a normal truth degree. 
One potential solution is to transform the formula into Prenex Normal Form (PNF) $\forall u \forall v (P(u) \lor Q(v))$, but this adds considerable complexity as we now need to ground combinations of individuals from $u$ and $v$. 
PNF formulas that contain universal quantifiers within the scope of existential quantifiers, such as $\forall u \exists v \forall w P(u,v,w)$, are even more difficult to handle. 
\logLTN, on the other hand, is simpler to work with in these cases.

\subsection{Simplification of Log Sigmoid and Log Softmax}
\label{appendix:logsigmoid}
Given the sigmoid function $\sigmoid(x) = \frac{1}{1+e^{-x}}$, for large negative values of $x$, we have $\log(\sigmoid(x)) = \log(\frac{1}{1+N}) = \log(1)-\log(1+N) \approx -N$ where $N$ is a large number. 
The derivative is also simple. Given that $\partialderivative{\sigmoid(x)}{x} = \sigmoid(x) (1-\sigmoid(x))$, we have:
\begin{equation}
    \partialderivative{\log(\sigmoid(x))}{x} = 1-\sigmoid(x)
\end{equation}
Similarly, given that the softmax function $\softmax(\mathbf{z})_i = \frac{e^{z_i}}{\sum_{j=1}^{K} e^{z_j}}$ over a vector $\mathbf{z}$ of $K$ values, $i=1,\dots,K$, 
has the derivatives $\partialderivative{\softmax(\mathbf{z})_i}{\mathbf{z}_j} = \softmax(\mathbf{z})_i (\delta_{ij}- \softmax(\mathbf{z})_j)$, we have:
\begin{equation}
    \partialderivative{\log(\softmax(\mathbf{z})_i)}{\mathbf{z}_j} = \delta_{ij}- \softmax(\mathbf{z})_j
\end{equation}
where $\delta_{ij} = \begin{cases}
        1 \ \ i=j\\
        0 \ \ i \neq j 
    \end{cases}$.

\section{Experiments}
\subsection{Semantic PASCAL-Part Dataset}
\label{appendix:pascalpart}
The semantic PASCAL-Part dataset is a simplified version of the PASCAL-Part dataset introduced by \cite{pascalpart_chen_2014}. 
The goal is to train a type classifier $\istype(x,\bottle)$, $\istype(x,\captype)$, etc., that predicts the type of an object within a bounding box $x$,
 and to train a relation predictor $\ispartof(x,y)$ that determines if one bounding box $x$ is part of another bounding box $y$.

\subsubsection{Constraints}
Training is guided by three constraints based on ground truth examples,
one for labeled type examples, one for pairs of positive examples for $\ispartof$, and one for pairs of negative examples. 
Note that the negative pairs are always sampled in bounding boxes belonging to the same image.
\begin{gather}
    \forall \mathrm{diag}(x_\mathsf{label}, \mathsf{label})\ \istype(x_\mathsf{label}, \mathsf{label}) \\
    \forall \mathsf{pairs_+}\ \ispartof(\mathsf{pairs_+[0]}, \mathsf{pairs_+[1]}) \\
    \forall \mathsf{pairs_-}\ \lnot \ispartof(\mathsf{pairs_-[0]}, \mathsf{pairs_-[1]}) \\
\end{gather}
Where $\mathrm{diag}(x_\mathsf{label}, \mathsf{label})$ is a special quantification that aggregates only arranged pairs of bounding boxes and their labels,
$\mathsf{pairs_+}$ is a batch of positive examples of $\ispartof$, $\mathsf{pairs_-}$ is a batch of negative examples of $\ispartof$.

There are two constraints stating that $\ispartof$ is antisymmetric and antireflexive.
\begin{gather}
    \forall \mathsf{pairs}\ \lnot \big( \ispartof(\mathsf{pairs[0]}, \mathsf{pairs[1]}) \land \ispartof(\mathsf{pairs[1]}, \mathsf{pairs[0]}) \big) \\
    \forall x\ \lnot \ispartof(x, x)
\end{gather}

Finally, and most importantly, there are mereological constraints that the types and their meanings.
The mereological constraints are based on the ontologies in Table \ref{tab:ontologies} rearranged in the shape of \eqref{eq:constraint_mereo1_appendix} and \eqref{eq:constraint_mereo2_appendix}.
\begin{gather}
    \forall x,y \ \istype(x,\bottle) \land \ispartof(y,x) \imp (\istype(y,\captype) \lor \istype(x,\body)) \label{eq:constraint_mereo1_appendix} \\
    \forall x,y \ \istype(x,\captype) \land \ispartof(x,y) \imp \istype(y,\bottle) \label{eq:constraint_mereo2_appendix} 
\end{gather}

\subsubsection{Features}
This setup has been previously implemented by \cite{donadello_ltnimage_2017,van_krieken_semi-supervised_2019}. 
In these previous works, the bounding boxes were grounded with the object class predictions produced by an object detector trained on the PASCAL-Part dataset. 
This means that LTN was only used to correct the predictions of the detector. 
In contrast, we increased the difficulty by grounding the bounding boxes with a latent vector of 1024 features, which is produced by an intermediate layer of the FasterRCNN~\cite{Ren_fasterrcnn_2017}.
This means that LTN must also learn the final layers of the object classifier.
In addition, we included the coordinates of each bounding box and their overlapping ratio when grounding the pairs.

For a work that trains an object detector architecture end-to-end on all types using LTN, refer to \cite{manigrasso_faster-ltn_2021}.

\subsubsection{Violation metrics}
In addition to the standard accuracy metrics for each predictor (PR AUC for the part-of predictor and balanced accuracy for the type predictor), we also assess their combined performance by measuring \emph{violations} of the mereological constraints.
It's important to note that not all misclassifications are equal. For instance, if a model predicts a cap inside a plant, it suggests that the system has learned less from prior knowledge compared to a model that misclassifies the cap in a different context, such as a wrong bottle.

This concept is visually represented in Figure \ref{fig:examples_violations}. Our results demonstrate that not only does \logLTN\ exhibit significantly fewer violations, but it also reaches this outcome much faster compared to other baseline approaches, as shown in Figure \ref{fig:pascalpart_violations}.

\begin{figure}
    \centering
    \begin{subfigure}{0.45\textwidth}
      \centering
      \includegraphics[width=0.8\linewidth]{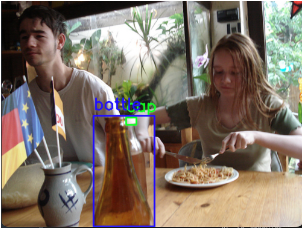}
      \caption{Nonviolation (a cap port of a bottle)}
      \hfill
    \end{subfigure}
    ~
    \begin{subfigure}{0.45\textwidth}
      \centering
      \includegraphics[width=0.8\linewidth]{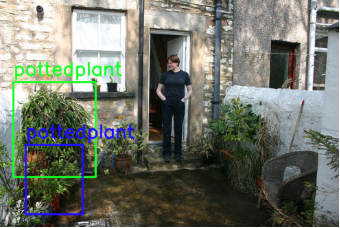}
      \caption{Violation (a plant part of another plant)}
    \end{subfigure}
    
    \caption{Examples of false positives on Semantic PASCAL-Part. Green is predicted in blue.}
    \label{fig:examples_violations}
  \end{figure}

\begin{figure}
    \centering
    \includegraphics[width=0.5\linewidth]{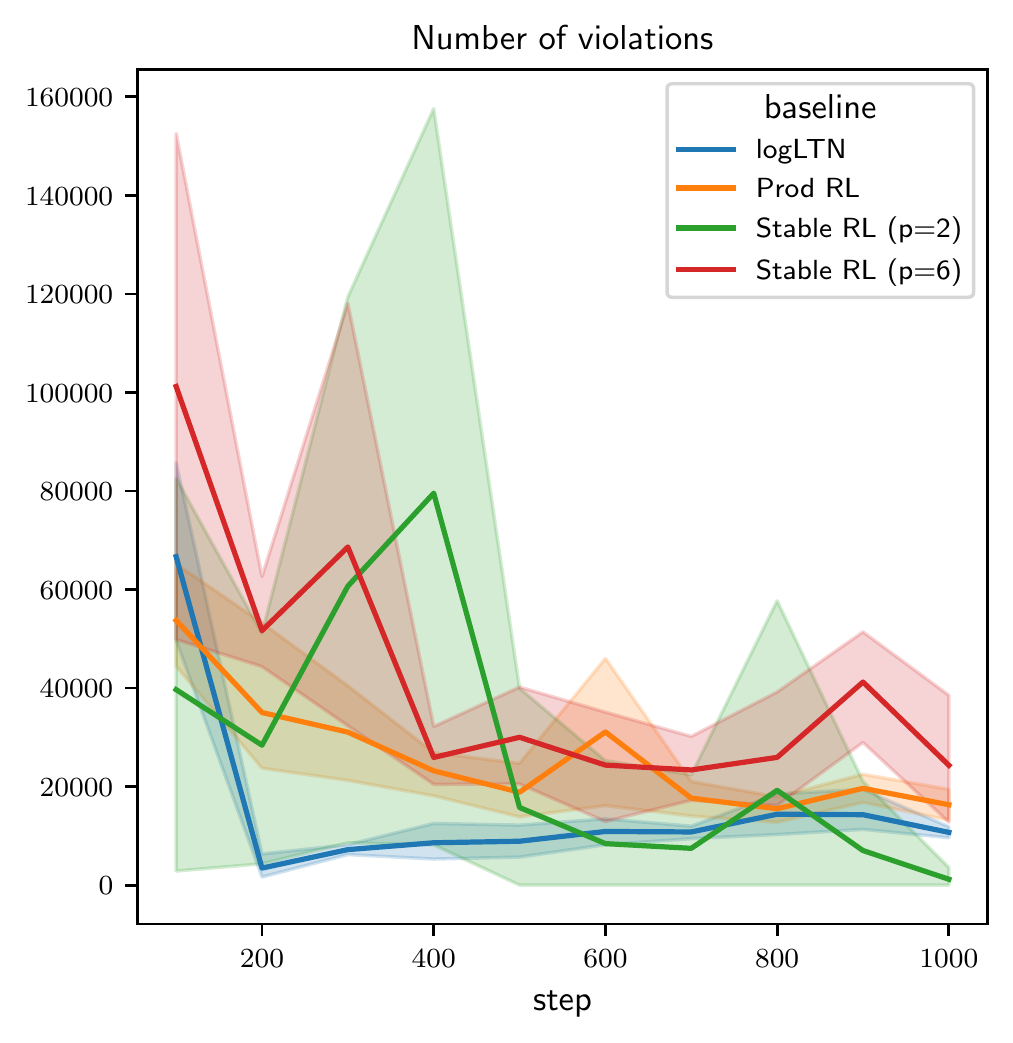}
    \caption{Violations over training time. The first 100 training steps are used for pretraining without the mereological constraints.}
    \label{fig:pascalpart_violations}
\end{figure}

\begin{table}
    \begin{tabular}{l|l}
        \toprule
        Whole & Parts \\
        \midrule
        aeroplane & artifact\_wing, body, engine, stern, wheel \\
        bicycle & chain\_wheel, handlebar, headlight, saddle, wheel \\
        bird & animal\_wing, beak, tail, eye, head, leg, neck, torso \\
        bottle & body, cap \\
        bus & bodywork, door, headlight, license\_plate, mirror, wheel, window \\
        car & bodywork, door, headlight, license\_plate, mirror, wheel, window \\
        cat & ear, tail, eye, head, leg, neck, torso \\
        cow & ear, horn, muzzle, tail, eye, head, leg, neck, torso \\
        dog & ear, muzzle, nose, tail, eye, head, leg, neck, torso \\
        horse & ear, hoof, muzzle, tail, eye, head, leg, neck, torso \\
        motorbike & handlebar, headlight, saddle, wheel \\
        person & arm, ear, ebrow, foot, hair, hand, mouth, nose, eye, head, leg, neck, torso \\
        pottedplant & plant, pot \\
        sheep & ear, horn, muzzle, tail, eye, head, leg, neck, torso \\
        train & coach, headlight, locomotive \\
        tvmonitor & screen \\
        boat & \\
        chair & \\
        sofa & \\
        diningtable & \\
        \bottomrule
    \end{tabular}
\caption{Full ontologies in Semantic PASCAL-Part}
\label{tab:ontologies}        
\end{table}

\subsection{Implementation Details}
\label{appendix:implementation}

logLTN is made available as a subpackage of the LTN library.\footnote{\url{https://github.com/logictensornetworks/logictensornetworks}}
Table \ref{tab:appendix_models} details the neural models used in each experimental task and Table \ref{tab:appendix_operator_params} details the hyperparameters of the baselines.
The Adam optimizer is trained with a learning rate of $0.002$ in the clustering task and a learning rate of $0.001$ for the MNISTAdd and Semantic Image Interpretation (SII) task.
We run our experiments on a machine equipped with a Tesla T4 GPU.

\begin{table}
    \begin{center}
    \begin{tabular}{llll}
        \toprule
        Baseline  & Operator & Parameter & Schedule \\
        \midrule
        logLTN & $\exists: \LME$ & $\alpha$ & Linear: $[1,4]$\\
        Prod RL & $\exists: \mathrm{pM}$ & $p$ &  Linear: $[1,6]$\\
        Stable RL & $\exists: \mathrm{pM}$ & $p$ &  Linear: $[1,6]$\\
        \bottomrule
    \end{tabular}
    \end{center}
    
    \begin{tablenotes}
        \small
        \item Linear $[a,b]$: the parameter increases linearly from $a$ to $b$ over the training steps,
        \item $\mathrm{pM}(x_1, \dots, x_n \mid p) = (\frac{1}{n} \sum_{i=1}^{n} x_i^p)^{\frac{1}{p}}$,
        \item $\mathrm{pME}(x_1, \dots, x_n \mid p) = 1 - (\frac{1}{n} \sum_{i=1}^{n} (1- x_i)^p)^{\frac{1}{p}}$.
      \end{tablenotes}
    \caption{Overview of the hyperparameters used in the experiments for each fuzzy operator configuration.}
    \label{tab:appendix_operator_params}
\end{table}

\newcommand{\Conv}{\mathrm{Conv}}
\newcommand{\Dense}{\mathrm{Dense}}

\begin{table*}[b]
    \centering
    \begin{tabular}{llll}
        \toprule
        Task  & Predicate & Model & Output layer \\
        \midrule
        Clustering & $C(x,c)$ & $\mathrm{Dense}(16)^\ast$, $\mathrm{Dense}(16)^\ast$, $\mathrm{Dense}(16)$ & $\mathrm{Softmax}$ \\
        MNISTAdd & $\isdigit(x,d)$ & \makecell[tl]{$\Conv(6,5)^\ast$, $\mathrm{MP}(2,2)$, $\Conv(16,5)^\ast$, $\mathrm{MP}(2,2)$,\\
                 $\Dense(100)^\ast$, $\Dense(84)^\ast$, $\Dense(10)$} & $\mathrm{Softmax}$\\
        SII & $\mathsf{type}(x,t)$ & \makecell[tl]{$\mathrm{Dense}(512)^\ast$, $\mathrm{Dense}(256)^\ast$, $\mathrm{Dense}(256)^\ast$,\\ 
                $\mathrm{Dense}(128)^\ast$, $\mathrm{Dense}(128)$} & $\mathrm{Softmax}$ \\
        SII & $\mathsf{partof}(x,y)$ & \makecell[tl]{$\mathrm{Concat}(\mathsf{type}_{\mathrm{Model}}(x),\mathsf{type}_{\mathrm{Model}}(y))$, $\mathrm{Dense}(512)^\ast$, \\ 
                $\mathrm{Dense}(256)^\ast$, $\mathrm{Dense}(256)^\ast$, $\mathrm{Dense}(128)^\ast$, $\mathrm{Dense}(128)$}      &  $\mathrm{Sigmoid}$ \\
        \bottomrule
    \end{tabular}
    \begin{tablenotes}
        \small
        \item $\ast$: layer ends with an $\mathsf{elu}$ activation,
        \item $\mathrm{Dense}(k)$: linear layer with $k$ units,
        \item $\Conv(f,k)$ : 2D convolution layer with $f$ filters and a kernel of size $k$,
        \item $\mathrm{MP}(w,h)$ : max pooling operation with a $w \times h$ pooling window.
      \end{tablenotes}
    \caption{Overview of the neural architectures used in each task.}
    \label{tab:appendix_models}
\end{table*}

\section{Theory}
\subsection{Log-Negations}
\subsubsection{Proof of Theorem \ref{thm:log_not_sigmoid}}
\label{appendix:proof_log_not_sigmoid}
\lognotsigmoid*
\begin{proof}
    \begin{align*}
        (\log \circ \Ns) (x) &= \log(1-S(x)) = \log(S(x))+\log(\frac{1-S(x)}{S(x)})\\
        & =\log(S(x)) + \log(\frac{1 - \frac{e^x}{e^x+1}}{ \frac{e^x}{e^x+1}}) \\
        & = \log(S(x)) + \log(\frac{e^x + 1 - e^x}{e^x})\\
        & = \log(S(x)) - x \qedhere
    \end{align*}
\end{proof}
    
\subsubsection{Proof of Theorem \ref{thm:log_not_softmax}}
\label{appendix:proof_log_not_softmax}
\lognotsoftmax*
\begin{proof}
    \begin{align*}
        (\log \circ \Ns) (\softmax(\mathbf{z})_i)
         &= \log(1-\softmax(\mathbf{z})_i) \\
         & = \log(\softmax(\mathbf{z})_i) + \log(\frac{1-\softmax(\mathbf{z})_i}{\softmax(\mathbf{z})_i})\\
         &=  \log(\softmax(\mathbf{z})_i) + \log(\frac{1 - \frac{e^{z_i}}{\sum_{j=1}^{K} e^{z_j}}}{\frac{e^{z_i}}{\sum_{j=1}^{K} e^{z_j}}})\\
        & =  \log(\softmax(\mathbf{z})_i) + \log(\frac{\sum_{j=1}^{K} e^{z_j} - e^{z_i}}{e^{z_i}}) \\
        &= \log(\softmax(\mathbf{z})_i) + \log(\sum_{\substack{j=1\\j\neq i}}^K e^{z_j}) - z_i \qedhere
    \end{align*}
\end{proof}

\subsection{Relaxation of the Disjunction}
\subsubsection{Bounds of LogSumExp}
\label{appendix:bounds_LSE_lower}
We start from Equation \eqref{eq:bounds_LSE_traditional} and substract $-\frac{\log(n)}{a}$ in all parts of the inequality:
\begin{equation}
    \label{eq:appendix_LSE_lb_1}
    \max(\mathbf{x}) - \frac{\log(n)}{\alpha} \leq \LSE(\mathbf{x} \mid \alpha, C) - \frac{\log(n)}{\alpha} \leq \max(\mathbf{x})
\end{equation}

And:
\begin{align}
    \LSE(\mathbf{x} \mid \alpha, C) - \frac{\log(n)}{\alpha} &= \frac{1}{\alpha} \Bigg( C + \log(\sum_{i=1}^n e^{\alpha x_i - C}) - \log(n)\Bigg) \\
            & \quad = \frac{1}{\alpha} \left( C + \log(\frac{\sum_{i=1}^n e^{\alpha x_i-C}}{n} ) \right) \\
   \label{eq:appendix_LSE_lb_2}         & \quad = \LME(\mathbf{x} \mid \alpha, C)
\end{align} 

Given \eqref{eq:appendix_LSE_lb_1} and \eqref{eq:appendix_LSE_lb_2}, we obtain the inequality for $\LME$:
\begin{equation}
    \max(\mathbf{x}) - \frac{\log(n)}{\alpha} \leq \LME(\mathbf{x} \mid \alpha, C) \leq \max(\mathbf{x})
\end{equation}

\subsection{De Morgan's Inequalities}
\subsubsection{Proof of Theorem \ref{thm:log_demorgans}}
\label{appendix:proof_nnfdemorgans}
\logdemorgans*

\begin{proof}
    Let $\G(P) = x$, $\G(Q) = y$, and $\G(P(u)) = [x_1, \dots, x_n]$, where $x, y, x_1, \dots, x_n \in [0,1]$.
    Grounding the operators with the definitions of Section \ref{s:logltn_semantics}, we need to prove:
    \begin{align}
        \Ns(\Tprod(x,y)) &\geq \Smax(\Ns(x),\Ns(y)) \label{eq:demorgan_tprod}\\
        \Ns(\Smax(x,y)) &\geq \Tprod(\Ns(x),\Ns(y)) \label{eq:demorgan_smax}\\
        \Ns(\Atprod(x_1,\cdots,x_n)) &\geq \Asmax(\Ns(x_1),\dots,\Ns(x_n)) \label{eq:demorgan_atprod}\\
        \Ns(\Asmax(x_1,\cdots,x_n)) &\geq \Atprod(\Ns(x_1),\dots,\Ns(x_n)) \label{eq:demorgan_asmax}
    \end{align}
    \eqref{eq:demorgan_tprod} and \eqref{eq:demorgan_smax} are specializations of the other two equations by working with $n=2$ values.
    We focus on proving \eqref{eq:demorgan_atprod}.
    Indeed, we can easily retrieve \eqref{eq:demorgan_asmax} from \eqref{eq:demorgan_atprod}.
    \paragraph{Equivalence of \eqref{eq:demorgan_atprod} and \eqref{eq:demorgan_asmax}}
    Posing $x_i'=1-x_i$ for $i=1\dots n$, we have:
    \begin{align}
        & \quad \Ns(\Atprod(x_1,\cdots,x_n)) \geq \Asmax(\Ns(x_1),\dots,\Ns(x_n))\\
        \iff & \quad \Ns(\Ns(\Atprod(x_1,\cdots,x_n))) \leq \Ns(\Asmax(\Ns(x_1),\dots,\Ns(x_n))) \\
        \iff & \quad \Atprod(x_1,\cdots,x_n) \leq \Ns(\Asmax(\Ns(x_1),\dots,\Ns(x_n))) \\
        \iff & \quad \Atprod(\Ns(x_1'),\cdots,\Ns(x_n')) \leq \Ns(\Asmax(x_1',\dots,x_n'))\\
        \iff & \quad \Ns(\Asmax(x_1',\dots,x_n')) \geq \Atprod(\Ns(x_1'),\cdots,\Ns(x_n'))
    \end{align}

    \paragraph{Proof of \eqref{eq:demorgan_atprod}}
    In the left-hand side of the inequality, we have:
    \begin{equation}
        \Ns(\Atprod(x_1,\cdots,x_n)) = 1-\prod_{i=1}^{n} x_i
    \end{equation}
    and in the right-hand side:
    \begin{equation}
        \Asmax(\Ns(x_1),\dots,\Ns(x_n)) = \max_{i=1}^n (1-x_i) = 1-\min_{i=1}^n (x_i)
    \end{equation}
    Replacing them in the original inequality, and denoting $j = \argmin_{i=1}^n (x_i)$, we obtain:
    \begin{align}
        & \quad 1-\prod_{i=1}^{n} x_i \geq 1-\min_{i=1}^n (x_i) \\
        \iff & \quad \prod_{i=1}^{n} x_i \leq \min_{i=1}^n (x_i) \\
        \iff & \quad x_j \prod_{\substack{i=1\\i\neq j}}^n x_i \leq x_j \\
        \iff & \quad \prod_{\substack{i=1\\i\neq j}}^n x_i \leq 1
    \end{align}
    Which is true because all $x_i \in [0,1]$.
\end{proof}

\subsubsection{Tightness of the bounds}
\label{appendix:bound_tightness}
We provide an analysis of the tightness of the bounds of the De Morgan's inequalities.
We measure the tightness for the quantifier variants of the inequalities, as this generalizes to the case $n=2$.
\paragraph{Tightness of \eqref{eq:demorgan_atprod}}
Let us first characterize the maximum value of the bound.
We are interesting in finding the values $(x_1^*,\dots,x_n^*)$ that maximize the difference of the two members in \eqref{eq:demorgan_atprod}.
\begin{align}
    (x_1^*,\dots,x_n^*) = \argmax_{(x_1,\dots,x_n)}  \Delta_\land(x_1,\dots,x_n)
\end{align}
with
\begin{align}
     \Delta_\land(x_1,\dots,x_n) &= \Ns(\Atprod(x_1,\cdots,x_n)) - \Asmax(\Ns(x_1),\dots,\Ns(x_n)) \\
            &= (1-\prod_{i=1}^{n} x_i) - (1-\min_{i=1}^n(x_i) ) = \min_{i=1}^n(x_i) - \prod_{i=1}^n x_i 
\end{align}
Let $j = \argmin_{i=1}^n (x_i)$.
\begin{align}
    \Delta_\land(x_1,\dots,x_n) = x_j - \prod_{i=1}^n x_i =  x_j (1-\prod_{\substack{i=1\\i\neq j}}^n x_i)
\end{align}
For any set of values, we have $\Delta_\land(x_j,\dots,x_j) \geq \Delta_\land(x_1,\dots,x_n)$, 
as $(1-x_j^{n-1}) \geq (1-\prod_{\substack{i=1\\i\neq j}}^n x_i)$ given that $x_j = \min_{i=1}^n (x_i)$.
Therefore, $(x_1^*,\dots,x_n^*)$ is actually a tuple of the same value taken $n$ times,
and we reduce the search to:
\begin{align}
    x^* = \argmax_{x}  \Delta_\land(x) = \argmax_x x - x^n
\end{align}
Given that $x \in [0,1]$, we find by first and second order derivative analysis that $\Delta_\land(x)$ is concave on the whole domain and has a single maximum at:
\begin{align}
    \label{eq:demorgan_and_max_bound}
    x^* & = n^{-(\frac{1}{n-1})} 
\end{align}

For $n=2$, that is, when applying the De Morgan's law to a simple conjunction of two terms $x_1$ and $x_2$, the bound of the inequality is maximal when $x_1=x_2=0.5$, giving $\Delta_\land = 0.25$.
However, on average, the bound is smaller. 
By sampling 10e4 $\times$ 10e4 points linearly on the domain $[0,1] \times [0,1]$, we find an average $\Delta_\land = 0.083167$. 
We visualize the bound of the inequality in Figure \ref{fig:demorgan_bounds}.
It is zero when any $x_i=0$ or $x_i=1$, in which case the De Morgan's laws are verified.

For larger $n$, that is, when applying the quantifier equivalent of the De Morgan's law, the maximal value of the bound becomes larger.
For example, with $n=8$, the bound is maximal for $x_1=\dots=x_8\approx 0.743$, with $\Delta_\land \approx 0.650$.
However, on average, the bound stays small.
By sampling 10e8 points linearly in the domain $[0,1]^8$, we have an average bound of $0.07135$.

\begin{figure}
    \centering
    \includegraphics[width=0.8\textwidth]{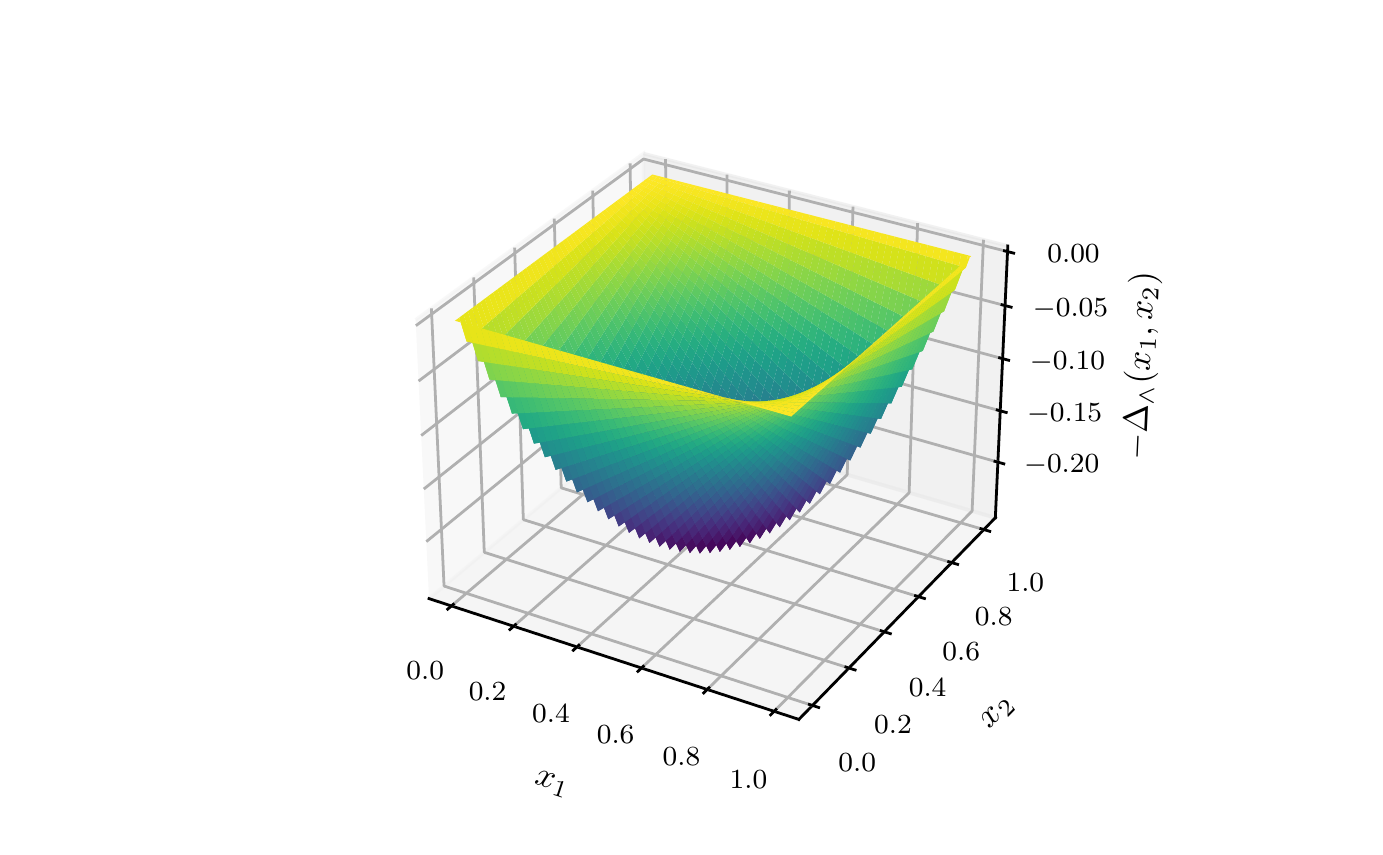}
    \caption{De Morgan's inequality bound $-\Delta_\land(x_1,x_2)$}
    \label{fig:demorgan_bounds}
\end{figure}

\paragraph{Tightness of \eqref{eq:demorgan_asmax}}
As we have shown in Section~\ref{appendix:proof_nnfdemorgans}, \eqref{eq:demorgan_asmax} is equivalent to \eqref{eq:demorgan_atprod} by replacing the input values with $x_i'=1-x_i$ for $i=1\dots n$.
Therefore, the maximum value of the bound is the same, except that it happens on a complement set of value.
We have:
\begin{align}
    x^* & = \argmax_{x}  \Delta_\lor(x) \\ 
        &= \argmax_{x} \Ns(\Asmax(x,\cdots,x)) - \Atprod(\Ns(x),\dots,\Ns(x)) \\
            &= \argmax_x 1 - x - (1-x)^{n} \\
            \label{eq:demorgan_or_max_bound}
            &= 1 - n^{-(\frac{1}{n-1})} 
\end{align}

For $n=2$, when applying the De Morgan's law to a simple disjunction of two terms $x_1$ and $x_2$, $\Delta_\lor$ is maximal when $x_1=x_2=0.5$.
We find again that the peak is $\Delta_\lor = 0.25$, and by sampling 10e4 $\times$ 10e4 points linearly on the domain $[0,1] \times [0,1]$, we find an average $\Delta_\lor =  0.083167$. 

For $n=8$, this time the bound is maximal when $x_1=\dots=x_8\approx 0.257$, with $\Delta_\lor \approx 0.650$.
Sampling 10e8 points linearly in $[0,1]^8$, we still get an average bound of $0.07135$.

\subsection{Common Fuzzy Properties for \logLTN}
We provide an overview of common fuzzy properties that are verified by the operator configuration in Table \ref{tab:fuzzy_properties}. 
The distributivity of $\land$ over $\lor$ for \logLTN\ is the only new property and can be demonstrated easily.
Let $\G(P)=x$, $\G(Q)=y$, and $\G(R)=z$ be the grounding of three predicates. We have:
\begin{align}
    & \G(P \land (Q \lor R)) = \G((P \land Q) \lor (P \land R)) \\
    \iff & \qquad  x \max(y,z) = \max(xy, xz)
\end{align}

\begin{table}
    \centering
    \begin{tabular}{l|c|c|c}
        \toprule
        Property                                            & \logLTN       & Prod RL       & Stable RL \\
        \midrule
        Commutativity of $\land$, $\lor$                    & \checkmark    & \checkmark    &  \checkmark \\
        Associativity of $\land$, $\lor$                    & \checkmark    & \checkmark    &  \checkmark \\
        De Morgan's laws for $\land$ and $\lor$             &               & \checkmark    &  \checkmark \\
        Material Implication                                & \checkmark    & \checkmark    & \checkmark  \\
        Distributivity of $\land$ over $\lor$               & \checkmark    &               &             \\
        Distributivity of $\lor$ over $\land$               &               &               &             \\
        Double negation, i.e.\ $\G(\lnot \lnot p) = \G(p)$                 & \checkmark    & \checkmark    & \checkmark  \\
        Law of non-excluded middle, i.e.\ $\G(p \land \lnot p) = 0$                          &               &               &             \\
        Law of non-contradiction, i.e.\ $\G(p \lor \lnot p) = 1$                           &               &               &             \\
        Conjunction elimination, i.e.\ $\G(p \land p) \leq \G(p)$ & \checkmark & \checkmark & \checkmark \\
        Disjunction amplification, i.e.\ $\G(p \lor p) \geq \G(p)$ & \checkmark & \checkmark & \checkmark \\
        $\forall$ defined as a generalization of $\land$    & \checkmark    & \checkmark    &             \\
        $\exists$ defined as a generalization of $\lor$     & \checkmark    &               &             \\
        De Morgan's laws for $\forall$ and $\exists$        &               &               &             \\
        
        \bottomrule
    \end{tabular}
    \caption{Fuzzy properties for \logLTN, Product Real Logic (Prod RL), and Stable Product Real Logic (Stable RL).}
    \label{tab:fuzzy_properties}
\end{table}

\end{document}